\documentclass[conference]{IEEEtran}

\usepackage{amsmath,amssymb}
\usepackage[mathscr]{euscript} 
\usepackage{amsthm} 
\usepackage{tikz} 
\usepackage{graphicx} 
\newcommand*{\Scale}[2][4]{\scalebox{#1}{\ensuremath{#2}}} 
\usepackage{framed} 
\usepackage{multirow}
\usepackage{color}
\usepackage{subfigure}
\usepackage{longtable}
\usepackage{MnSymbol}
\usepackage[numbers,square, comma, sort&compress]{natbib}	 
\usepackage[hidelinks, breaklinks=true]{hyperref} 

\newtheorem{myTheorem}{Theorem}
\newtheorem{myLemma}{Lemma}
\newtheorem{myCorollary}{Corollary}

\newtheorem{myExample}{Example}

\tikzstyle{every edge}=  [draw]
\tikzstyle{vertex} = [draw,circle,minimum size=1pt]
\tikzstyle{label} = [minimum size=.1pt,font=\scriptsize]
\tikzstyle{title} = [minimum size=.25cm,font=\small]

\newcommand{\bs}[1]{\boldsymbol{#1}}

\newcommand{\fix}[1]{\Scale[.7]{#1}}

\def \R{\mathbb{R}}
\def \P{\mathsf{P}}
\def \T{\mathsf{T}}
\def \Ord{\mathscr{O}}
\def \nuu{{\hyperref[nuuDef]{\nu}}}

\def \r{{\hyperref[rDef]{r}}}
\def \d{{\hyperref[dDef]{d}}}
\def \N{{\hyperref[NDef]{N}}}
\def \m{{\hyperref[mDef]{m}}}
\def \mOf{{\hyperref[mOfDef]{m}}}
\def \mp{{\hyperref[mpDef]{m'}}}
\def \mdp{{\hyperref[mdpDef]{m''}}}
\def \n{{\hyperref[nDef]{n}}}
\def \nOf{{\hyperref[nOfDef]{n}}}
\def \np{{\hyperref[npDef]{n'}}}
\def \ndp{{\hyperref[ndpDef]{n''}}}
\def \li{{\hyperref[liDef]{\aleph}}}
\def \liOf{{\hyperref[liOfDef]{\aleph}}}
\def \f{{\hyperref[polyEq]{f}}}
\def \L{{\hyperref[LDef]{\ell}}}
\def \Li{{\hyperref[LiDef]{\ell_i}}}
\def \eps{{\hyperref[epsDef]{\epsilon}}}
\def \Ntilde{{\hyperref[NtildeDef]{\widetilde{N}}}}

\def \sstar{{\hyperref[sstarDef]{S^\star}}}
\def \sstaroi{{\hyperref[soiDef]{S^\star_{\bs{\omega}_i}}}}
\def \soi{{\hyperref[soiDef]{S_{\bs{\omega}_i}}}}
\def \sstaro{{\hyperref[cdotoDef]{S^\star_{\bs{\upsilon}}}}}
\def \SS{{\hyperref[SSDef]{\mathscr{S}}}}
\def \Gr{\hyperref[GrDef]{{\rm Gr}}}
\def \s{{\hyperref[sDef]{S}}}
\def \so{{\hyperref[cdotoDef]{S_{\bs{\upsilon}}}}}

\def \I{{\hyperref[IDef]{\bs{{\rm I}}}}}
\def \one{{\hyperref[oneDef]{\bs{{\rm 1}}}}}
\def \ai{{\hyperref[aiDef]{\bs{{\rm a}}_i}}}
\def \aiT{{\hyperref[aiDef]{\bs{{\rm a}}^\T_i}}}

\def \aoT{{\hyperref[cdotoDef]{\bs{{\rm a}}^\T_{\bs{\upsilon}}}}}
\def \aoi{{\hyperref[aoiDef]{\bs{{\rm a}}_{\bs{\omega}_i}}}}
\def \aoiT{{\hyperref[aoiDef]{\bs{{\rm a}}^\T_{\bs{\omega}_i}}}}
\def \hatai{{\hyperref[hataiDef]{\bs{\hat{{\rm a}}}_i}}}
\def \AA{{\hyperref[AADef]{\bs{{\rm A}}}}}
\def \A{{\hyperref[ADef]{\bs{{\rm A}}'}}}
\def \Adp{{\hyperref[AdpDef]{\bs{{\rm A}}''}}}
\def \Ao{{\hyperref[AoDef]{\bs{{\rm A}}'_{\bs{\upsilon}}}}}
\def \AoT{{\hyperref[AoDef]{\bs{{\rm A}}'^\T_{\bs{\upsilon}}}}}
\def \AAstar{{\hyperref[AAstarDef]{\bs{{\rm A}}^\star}}}
\def \B{\hyperref[firstDecompositionAEq]{\bs{{\rm B}}}}
\def \Bone{\hyperref[BoneDef]{\bs{{\rm B}}_1}}
\def \BoneInv{\hyperref[BoneDef]{\bs{{\rm B}}_1^{-1}}}
\def \Btwo{\hyperref[BtwoDef]{\bs{{\rm B}}_2}}
\def \BtwoTilde{\hyperref[BtwoTildeDef]{\widetilde{\bs{{\rm B}}}_2}}
\def \C{\hyperref[firstDecompositionAEq]{\bs{{\rm C}}}}
\def \DDstar{{\hyperref[AAstarEq]{\bs{{\rm D}}^\star}}}
\def \bb{\hyperref[bbDef]{\bs{\beta}}}
\def \uoi{\hyperref[cdotoDef]{\bs{{\rm u}}_{\bs{\omega}_i}}}
\def \uo{\hyperref[cdotoDef]{\bs{{\rm u}}_{\bs{\upsilon}}}}
\def \xoi{{\hyperref[xoiDef]{\bs{{\rm x}}_{\bs{\omega}_i}}}}
\def \X{{\hyperref[XDef]{\bs{{\rm X}}}}}
\def \x{{\hyperref[xDef]{\bs{{\rm x}}}}}
\def \XO{{\hyperref[XODef]{\bs{{\rm X}}_{\bs{\Omega}}}}}
\def \XOone{{\hyperref[XOoneDef]{\bs{{\rm X}}_{\bs{\Omega}_1}}}}
\def \XOtwo{{\hyperref[XOtwoDef]{\bs{{\rm X}}_{\bs{\Omega}_2}}}}

\def \oi{{\hyperref[oiDef]{\bs{\omega}_i}}}
\def \OO{{\hyperref[OODef]{\bs{\Omega}}}}
\def \O{{\hyperref[ODef]{\bs{\Omega}'}}}
\def \OTilde{{\hyperref[OTildeDef]{\widetilde{\bs{\Omega}}}}}
\def \OOtwo{{\hyperref[XOtwoDef]{\bs{\Omega}_2}}}
\def \Op{{\hyperref[OpDef]{\bs{\Omega}'}}}
\def \Odp{{\hyperref[OdpDef]{\bs{\Omega}''}}}
\def \o{{\hyperref[oDef]{\bs{\upsilon}}}}

\newcommand{\Oi}[1]{{\hyperref[OiDef]{\bs{\Omega}_{#1}}}}
\def \OSplit{{\hyperref[OSplitDef]{\bs{\breve{\Omega}}}}}

\newcommand{\upsiBotNum}[1]{{\hyperref[upsiTopDef]{\bs{\medtriangledown}_{i #1}}}}
\newcommand{\oij}[1]{{\hyperref[oijDef]{\bs{\omega}_{i #1}}}}
\def \OHat{{\hyperref[OHatDef]{\widehat{\bs{\Omega}}}}}

\def \i{{\hyperref[iDef]{i}}}
\def \j{{\hyperref[jDef]{j}}}
\newcommand{\ki}[1]{{\hyperref[kiDef]{k_{#1}}}}

\def \E{{\hyperref[EDef]{\mathcal{E}}}}
\def \En{{\hyperref[EnDef]{\mathcal{E}_n}}}

\def \GOO{{\hyperref[GODef]{\mathcal{G}(\bs{\Omega})}}}

\def \J{{\hyperref[JDef]{\bs{\Upsilon}}}}
\def \Jp{{\hyperref[JpDef]{\bs{\Upsilon}'}}}
\def \Jdp{{\hyperref[JdpDef]{\bs{\Upsilon}''}}}

\def \mld{{\hyperref[mldDef]{minimally linearly dependent}}}
\def \ae{{\hyperref[aeDef]{{\rm a.e.}}}}
\def \LRMC{{\hyperref[LRMCDef]{LRMC}}}
\def \fits{{\hyperref[fitsDef]{fits}}}
\def \fit{{\hyperref[fitsDef]{fit}}}
\def \whp{{\hyperref[whpDef]{w.h.p}}}
\def \neighborhood{{\hyperref[neighborhoodDef]{neighborhood}}}

\def \AoneAss{{\hyperref[AoneAssDef]{\textbf{A1}}}}
\def \AonepAss{{\hyperref[AonepAssDef]{\textbf{A1'}}}}
\def \AonedpAss{{\hyperref[AonedpAssDef]{\textbf{A1''}}}}
\def \AtwoAss{{\hyperref[AtwoAssDef]{\textbf{A2}}}}
\def \AthreeAss{{\hyperref[AthreeAssDef]{\textbf{A3}}}}
\def \identifiabilityThm{{\hyperref[identifiabilityThm]{Theorem \ref{identifiabilityThm}}}}
\def \probabilityThm{{\hyperref[probabilityThm]{Theorem \ref{probabilityThm}}}}
\def \aEntriesLem{{\hyperref[aEntriesLem]{Lemma \ref{aEntriesLem}}}}
\def \independenceLem{{\hyperref[independenceLem]{Lemma \ref{independenceLem}}}}
\def \liLem{{\hyperref[liLem]{Lemma \ref{liLem}}}}
\def \basisLem{{\hyperref[basisLem]{Lemma \ref{basisLem}}}}
\def \LRMCnecCor{{\hyperref[LRMCnecCor]{Corollary \ref{LRMCnecCor}}}}
\def \LRMCsuffCor{{\hyperref[LRMCsuffCor]{Corollary \ref{LRMCsuffCor}}}}
\def \fitsCor{{\hyperref[fitsCor]{Corollary \ref{fitsCor}}}}
\def \rConnectedCor{{\hyperref[rConnectedCor]{Corollary \ref{rConnectedCor}}}}

\def \introEg{{\hyperref[introEg]{Example \ref{introEg}}}}
\def \graphEg{{\hyperref[graphEg]{Example \ref{graphEg}}}}
\def \modelSec{{\hyperref[modelSec]{Section \ref{modelSec}}}}
\def \LRMCSec{{\hyperref[LRMCSec]{Section \ref{LRMCSec}}}}
\def \graphSec{{\hyperref[graphSec]{Section \ref{graphSec}}}}
\def \proofSec{{\hyperref[proofSec]{Section \ref{proofSec}}}}

\def \probabilityApx{{\hyperref[probabilityApx]{appendix}}}
\def \LRMCnecApx{{\hyperref[LRMCnecApx]{appendix}}}

\def \fitsApx{{\hyperref[fitsApx]{appendix}}}
\def \rConnectedApx{{\hyperref[rConnectedApx]{appendix}}}
\def \generalizationApx{{\hyperref[generalizationApx]{appendix}}}

\begin{document}
\title{Deterministic Conditions for Subspace Identifiability from Incomplete Sampling}

\author{\IEEEauthorblockN{Daniel L. Pimentel-Alarc\'on, Nigel Boston, Robert D. Nowak}
\IEEEauthorblockA{University of Wisconsin-Madison}
}

\maketitle

\begin{abstract}
Consider an $\r$-dimensional subspace of $\R^{\fix{\d}}$, $\r<\d$, and suppose that we are only given projections of this subspace onto small subsets of the canonical coordinates.  The paper establishes necessary and sufficient deterministic conditions on the subsets for subspace identifiability.  The results also shed new light on low-rank matrix completion.
\end{abstract}

\IEEEpeerreviewmaketitle

\section{Introduction}
\label{introSec}
Subspace identification arises in a wide variety of signal and information processing applications. 
In many cases, especially high-dimensional situations, it is common to encounter missing data.
Hence the growing literature concerning the estimation of low-dimensional subspaces and matrices from incomplete data in theory \cite{balzano, chi, mardani,candes,recht,kiraly,jain} and applications \cite{eriksson,he}.

This paper considers the problem of identifying an \phantomsection\label{rDef}$\r$-dimensional subspace of \phantomsection\label{dDef}$\R^{\fix{\d}}$ from projections of the subspace onto small subsets of the canonical coordinates.  The main contribution of this paper is to establish deterministic necessary and sufficient conditions on such subsets that guarantee that there is {\em only} one $\r$-dimensional subspace consistent with all the projections.  These conditions also have implications for low-rank matrix completion and related problems.

\subsection*{Organization of the paper}
In \modelSec\ we formally state the problem and our main results. We present the proof of our main theorem in \proofSec.   \LRMCSec\ illustrates the implications of our results for low-rank matrix completion. \graphSec\ presents the graphical interpretation of the problem and another necessary condition based on this viewpoint.

\vspace{.1cm}
\section{Model and main results}
\label{modelSec}
\vspace{.1cm}
Let \phantomsection\label{sstarDef}$\sstar$ denote an $\r$-dimensional subpace of $\R^{\fix{\d}}$.
Define \phantomsection\label{OODef}$\OO$ as a \phantomsection\label{NDef}$\d \times \N$ binary matrix and let \phantomsection\label{oiDef}$\oi$ denote the \phantomsection\label{iDef}$\i^{th}$ column of $\OO$.  The nonzero entries of $\oi$ indicate the canonical coordinates involved in the $\i^{th}$ projection.

Since $\sstar$ is $\r$-dimensional, the restriction of $\sstar$ onto $\L \leq \r$ coordinates will be $\R^{\fix{\L}}$ (in general), and hence such a projection will provide no information specific to $\sstar$. Therefore, without loss of generality (see the \generalizationApx\ for immediate generalizations) we will assume that:
\begin{itemize}
\phantomsection\label{AoneAssDef}
\item[\AoneAss]
$\OO$ has exactly $\r+1$ nonzero entries per column.
\end{itemize}
 
Given an $\r$-dimensional subspace \phantomsection\label{sDef}$\s$, let \phantomsection\label{soiDef}$\soi \subset \R^{\fix{\r}+1}$ denote the restriction of $\s$ to the nonzero coordinates in $\oi$.  The question addressed in this paper is whether the restrictions $\{\sstaroi\}_{\fix{\i}=1}^{\fix{\N}}$ uniquely determine $\sstar$.  This depends on the sampling pattern in $\OO$.

\begin{figure}
\centering
\includegraphics[width=3.45cm]{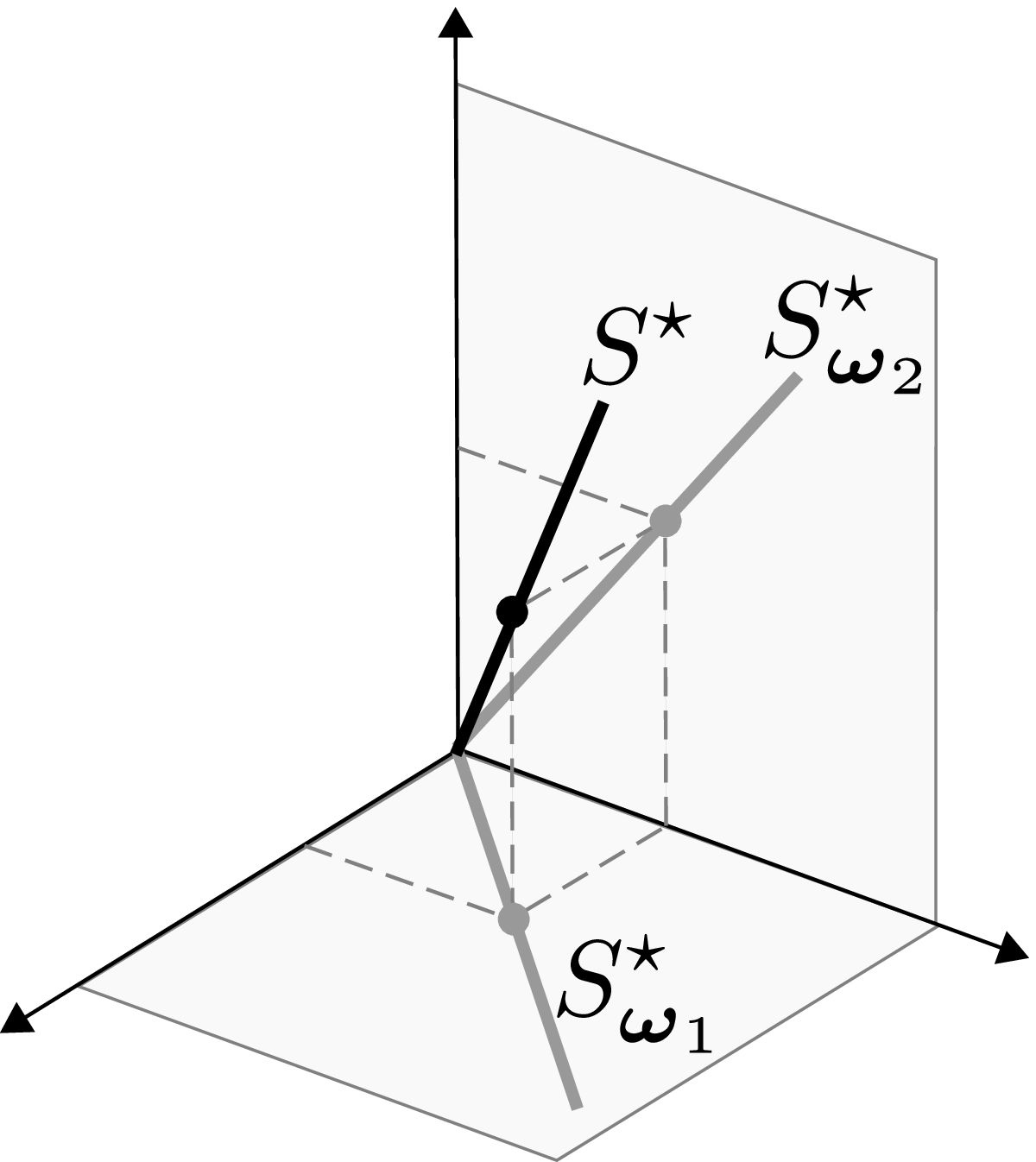}
\label{illustration}
\caption{When can $\sstar$ be identified from its canonical projections $\{\sstaroi\}_{\fix{\i}=1}^{\fix{\N}}$?}
\end{figure}

We will see that identifiability of this sort can only be possible if  $\N \geq \d-\r$, since $\ker \sstar$ is $(\d-\r)$-dimensional.  Thus, unless otherwise stated, we will also assume that:
\begin{itemize}
\phantomsection\label{AtwoAssDef}
\item[\AtwoAss]
$\OO$ has exactly $\N=\d-\r$ columns.
\end{itemize}

Let \phantomsection\label{GrDef}$\Gr(\r,\R^{\fix{\d}})$ denote the Grassmannian manifold of $\r$-dimensional subspaces in $\R^{\fix{\d}}$.  Define \phantomsection\label{SSDef}$\SS(\sstar,\OO)$ $\subset$ $\Gr(\r,\R^{\fix{\d}})$ such that every $\s \in \SS(\sstar,\OO)$ satisfies $\soi=\sstaroi$ $\forall$ $\i$.  In words, $\SS(\sstar,\OO)$ is the set of all $\r$-dimensional subspaces matching $\sstar$ on $\OO$.
\begin{myExample}
\label{introEg}
Let $\d=5$, $\r=1$,
\begin{align*}
\sstar \ = \ {\rm span}\left[\begin{matrix}
1 \\ 2 \\ 3 \\ 4 \\ 4
\end{matrix}\right] \hspace{.25cm} \text{and } \hspace{.25cm}
\OO \ = \ \left[\begin{matrix}
1 & 0 & 1 & 0 \\
1 & 1 & 0 & 0 \\
0 & 1 & 1 & 0 \\
0 & 0 & 0 & 1 \\
0 & 0 & 0 & 1 \\
\end{matrix}\right].
\end{align*}
Then, for example, 
\begin{align*}
\hyperref[soiDef]{S^\star_{\bs{\omega}_3}}={\rm span}\left[\begin{matrix} 1 \\ 3\end{matrix}\right].
\end{align*}
It is easy to see that there are infinitely many $1$-dimensional subspaces that match $\sstar$ on $\OO$.  In fact,
\begin{align*}
\SS(\sstar,\OO) \ = \ \bigg\{ {\rm span}[1 \ \ 2 \ \ 3 \ \ \alpha \ \ \alpha]^\T \ : \
\alpha \in \R \backslash \{0\} \bigg\}.
\end{align*}
However,  if we instead had $\hyperref[oiDef]{\bs{\omega}_3}=[0 \ \ 0 \ \ 1 \ \ 1 \ \ 0]^\T$, then $\sstar$ would be the only subspace in $\SS(\sstar,\OO)$.
\end{myExample}

The main result of this paper is the following theorem, which gives necessary and sufficient conditions on $\OO$ to guarantee that $\SS(\sstar,\OO)$ contains no subspace other than $\sstar$.  Our results hold for \phantomsection\label{aeDef}{\rm(\ae)} $\sstar$, with respect to the uniform measure over $\Gr(\r,\R^{\fix{\d}})$.

Given a matrix, let $\phantomsection\label{nOfDef}\nOf(\bs{\cdot})$ denote its number of columns, and \phantomsection\label{mOfDef}$\mOf(\bs{\cdot})$ the number of its {\em nonzero} rows.
\begin{framed}
\begin{myTheorem}
\label{identifiabilityThm}
Let \AoneAss\ and \AtwoAss\ hold.  For \hyperref[aeDef]{almost every} $\sstar$, $\sstar$ is the only subspace in $\SS(\sstar,\OO)$ if and only if every matrix \phantomsection\label{ODef}$\O$ formed with a subset of the columns in $\OO$ satisfies
\begin{align}
\label{identifiabilityEq}
\mOf(\O) \ \geq \ \nOf(\O) + \r.
\end{align}
\end{myTheorem}
\end{framed}

The proof of \identifiabilityThm\ is given in \proofSec.
In words, \identifiabilityThm\ is stating that $\sstar$ is the only subspace that matches $\sstar$ in $\OO$ if and only if every subset of $\nOf$ columns of $\OO$ has at least $\nOf+\r$ nonzero rows.

\begin{myExample}
The following matrix, where \phantomsection\label{oneDef}$\one$ denotes a block of all $1$'s and \phantomsection\label{IDef}$\I$ denotes the identity matrix, satisfies the conditions of \identifiabilityThm:
\begin{align*}
\OO \ = \ \left[ \begin{array}{c}
\hspace{.3cm} \Scale[1.5]{\one} \hspace{.3cm} \\ \hline
\\
\Scale[1.5]{\I} \\ \\
\end{array}\right]
\begin{matrix}
\left. \begin{matrix} \\ \end{matrix} \right\} \r \hspace{.7cm} \\
\left. \begin{matrix} \\ \\ \\ \end{matrix} \right\} \d-\r.
\end{matrix}
\end{align*}
\end{myExample}

When the conditions of \identifiabilityThm\ are satisfied, identifying $\sstar$ becomes a trivial task: $\sstar= \ker \AA{}^\T$, with $\AA$ as defined in \proofSec.

In general, verifying the conditions on $\OO$ in \identifiabilityThm\ may be computationally prohibitive, especially for large $\d$.  However, as the next theorem states, uniform random sampling patterns will satisfy the conditions in \identifiabilityThm\ with high probability \phantomsection\label{whpDef}(\whp.).  

\begin{myTheorem}
\label{probabilityThm}
Assume \AtwoAss\ and let \phantomsection\label{epsDef}$0<\eps \leq 1$ be given.  Suppose $\r \leq \frac{\fix{\d}{}}{6}$ and that each column of $\OO$ contains at least \phantomsection\label{LDef}$\L$ nonzero entries, selected uniformly at random and independently across columns, with
\begin{align}
\label{kEq}
\textstyle \L \ \geq \ \max \left\{9 \log(\frac{\fix{\d}{}}{\fix{\eps}{}})+12, \ 2\r \right\}.
\end{align}
Then $\OO$ will satisfy the conditions of \identifiabilityThm\ with probability at least $1-\eps$.
\end{myTheorem}

\probabilityThm\ is proved in the \probabilityApx. Notice that $\Ord(\r \log\d)$ nonzero entries per column is a typical requirement of \LRMC\ methods, while $\Ord(\max\{\r,\log\d\})$ is sufficient for subspace identifiability.

\vspace{.3cm}
\section{Proof of \identifiabilityThm}
\label{proofSec}
\vspace{.3cm}
For any subspace, matrix or vector that is compatible with a binary vector \phantomsection\label{oDef}$\o$, we will use the subscript \phantomsection\label{cdotoDef}$\o$ to denote its restriction to the nonzero coordinates/rows in $\o$.  For \ae\ $\sstar$, $\sstaroi$ is an $\r$-dimensional subspace of $\R^{\fix{\r}+1}$, and the kernel of $\sstaroi$ is a $1$-dimensional subspace of $\R^{\fix{\r}+1}$.

\begin{myLemma}
\label{aEntriesLem}
Let \phantomsection\label{aoiDef}$\aoi \in \R^{\fix{\r}+1}$ be a nonzero element of $\ker \sstaroi$.  All entries of $\aoi$ are nonzero for \ae\ $\sstar$.
\end{myLemma}

\begin{proof}
Suppose $\aoi$ has at least one zero entry.  Use $\o$ to denote the binary vector of the nonzero entries of $\aoi$.  Since $\aoi$ is orthogonal to $\sstaroi$, for every $\uoi \in \sstaroi$ we have that $\aoiT \uoi=\aoT \uo=0$.  Then \phantomsection\label{sstaroDef}$\sstaro$ satisfies
\begin{align}
\label{aEntriesEq}
\dim \sstaro \ \leq \ \dim \ker \aoT \ = \ \|\o\|_1-1 \ < \ \|\o\|_1.
\end{align}
Observe that for every binary vector $\o$ with $\|\o\|_1 \leq \r$, \ae\ $\r$-dimensional subspace $\s$ satisfies $\dim \so = \|\o\|_1$.  Thus \eqref{aEntriesEq} holds only in a set of measure zero.
\end{proof}

Define \phantomsection\label{aiDef}$\ai$ as the vector in $\R^{\fix{\d}}$ with the entries of $\aoi$ in the nonzero positions of $\oi$ and zeros elsewhere.  Then $\s \subset \ker \aiT$ for every $\s \in \SS(\sstar,\OO)$ and every $\i$.  Letting \phantomsection\label{AADef}$\AA$ be the $\d \times (\d-\r)$ matrix formed with $\{\ai\}_{\fix{\i}=1}^{\fix{\d}-\fix{\r}}$ as columns, we have that $\s \subset \ker \AA{}^\T$ for every $\s\in \SS(\sstar,\OO)$.
Note that if $\dim \ker \AA{}^\T=\r$, then $\SS(\sstar,\OO)$ contains just one element, $\sstar$, which is the identifiability condition of interest. Thus, we will establish conditions on $\OO$ guaranteeing that the $\d-\r$ columns of $\AA$ are linearly independent.

Recall that for any matrix \phantomsection\label{ADef}$\A$ formed with a subset of the columns in $\AA$, $\nOf(\A)$ denotes the number of columns in $\A$, and $\mOf(\A)$ denotes the number of {\em nonzero} rows in $\A$.

\begin{myLemma}
\label{independenceLem}
For \ae\ $\sstar$, the columns of $\AA$ are linearly dependent if and only if $\nOf(\A)>\mOf(\A)-\r$ for some matrix $\A$ formed with a subset of the columns in $\AA$.
\end{myLemma}

We will show \independenceLem\ using Lemmas \ref{liLem} and \ref{basisLem} below.  Let \phantomsection\label{liOfDef}$\liOf(\A)$ be the largest number of linearly independent columns in $\A$, i.e., the column rank of $\A$.

\begin{myLemma}
\label{liLem}
For \ae\ $\sstar$, $\liOf(\A) \leq \mOf(\A) - \r$.
\end{myLemma}

\begin{proof}
Let $\o$ be the binary vector of nonzero rows of $\A$, and \phantomsection\label{AoDef}$\Ao$ be the $\mOf(\A) \times \nOf(\A)$ matrix formed with these rows.

For \ae\ $\sstar$, $\dim \sstaro=\r$.  Since $\sstaro \subset \ker \AoT$, $\r=\dim \sstaro \leq \dim \ker \AoT = \mOf(\A)-\liOf(\A) $.
\end{proof}

We say $\A$ is \phantomsection\label{mldDef}{\em \mld} if the columns in $\A$ are linearly dependent, but every proper subset of the columns in $\A$ is linearly independent.

\begin{myLemma}
\label{basisLem}
Let $\A$ be \mld.  Then for \ae\ $\sstar$, $\nOf(\A) = \mOf(\A) - \r + 1$.
\end{myLemma}

\begin{proof}
Let \phantomsection\label{AdpDef}$\A=[ \ \Adp \ | \ \ai \ ]$ be \mld.
Let \phantomsection\label{mDef}$\m=\mOf(\Adp)$, \phantomsection\label{nDef}$\n = \nOf(\Adp)$, and \phantomsection\label{liDef}$\li=\liOf(\Adp)$.
Define \phantomsection\label{bbDef}$\bb \in \R^{\fix{\n}}$ such that
\begin{align}
\label{aLdOnAEq}
\Adp \bb \ = \ \ai \ .
\end{align}
Note that because $\A$ is \mld, all entries in $\bb$ are nonzero.
Since the columns of $\Adp$ are linearly independent, $\n=\li$. Thus, by \liLem, $\n \leq \m-\r$.  We want to show that $\n=\m-\r$, so suppose for contradiction that $\n<\m-\r$.

We can assume without loss of generality that $\Adp$ has all its zero rows (if any) in the first positions.  In that case, since $\A$ is \mld, it follows that the nonzero entries of $\ai$ cannot be in the corresponding rows.  Thus, without loss of generality, assume that $\ai$ has its first $\r$ nonzero entries in the first $\r$ nonzero rows of $\Adp$, and that the last nonzero entry of $\ai$ is $1$ (i.e., rescale $\ai$ if needed), and is located in the last row.  Let \phantomsection\label{hataiDef}$\hatai \in \R^{\fix{\r}}$ denote the vector with the first nonzero entries of $\ai$, such that we can write:
\begin{align}
\label{firstDecompositionAEq}
\left[\begin{array}{c|c}
\Adp & \ai
\end{array}\right]
=\left[\begin{matrix} \\ \\ \\ \\ \\ \\ \\ \\ \end{matrix}\right.
\underbrace{\begin{array}{c|}
\hspace{.5cm} \Scale[1.5]{\bs{0}} \hspace{.5cm} \\ \hline
\\ \Scale[1.5]{\bs{\C}} \\ \\ \hline
\\ \multirow{2}{*}{\Scale[1.5]{\bs{\B}}} \\ \\ \\
\end{array}}_{\fix{\n}{}}
\underbrace{
\begin{array}{c}
\vspace{.05cm} \Scale[1]{\bs{0}} \\ \hline
\\ \Scale[1]{\bs{\hatai}} \\ \vspace{.05cm} \\ \hline
\\ \Scale[1]{\bs{0}} \\ \\ \hline
1 \\
\end{array}}_{1}
\left.\begin{matrix} \\ \\ \\ \\ \\ \\ \\ \\ \end{matrix}\right]
\begin{matrix}
\left. \begin{matrix} \\ \end{matrix} \right\} \d-\m \hspace{.6cm} \\
\left. \begin{matrix} \\ \\ \\ \end{matrix} \right\} \r \hspace{1.2cm} \\
\left. \begin{matrix} \\ \\ \\ \end{matrix} \right\} \m-\r-1 \\
\left. \begin{matrix} \\ \end{matrix} \right\} 1, \hspace{1.1cm}
\end{matrix}
\end{align}
where $\C$ and $\B$ are submatrices used to denote the blocks of $\Adp$ corresponding to the partition of $\ai$.

The columns of $\B$ are linearly independent.  To see this, suppose for contradiction that they are not.  This means that there exists some nonzero $\bs{\gamma} \in \R^{\fix{\n}}$, such that $\B \bs{\gamma} = 0$.  Let $\bs{c}=\Adp \bs{\gamma} $ and note that only the $\r$ rows in $\bs{c}$ corresponding to the block $\C$ may be nonzero.  Let $\o$ denote the binary vector of these nonzero entries. Since $\sstar$ is orthogonal to every column of $\Adp$ and $\bs{c}$ is a linear combination of the columns in $\Adp$, it follows that $\sstaro \subset \ker \bs{c}^\T_{\fix{\o}}$. This implies that $\dim \sstaro \leq \dim \ker \bs{c}^\T_{\fix{\o}} = \|\o\|_1-1$.  As in the proof of \aEntriesLem, this implies that the columns of $\B$ are linearly dependent only in a set of measure zero.

Going back to \eqref{firstDecompositionAEq}, since the $\n$ columns of $\B$ are linearly independent and because we are assuming that $\n<\m-\r$, it follows that $\B$ has $\n$ linearly independent rows.  Let \phantomsection\label{BoneDef}$\Bone$ denote the $\n \times \n$ block of $\B$ that contains $\n$ linearly independent rows, and \phantomsection\label{BtwoDef}$\Btwo$ the $(\m-\n-\r) \times \n$ remaining block of $\B$.

Notice that the row of $\B$ corresponding to the $1$ in $\ai$ must belong to $\Bone$, since otherwise, we have that $\Bone \bb =0$, with $\bb$ as in \eqref{aLdOnAEq}, which implies that $\Bone$ is rank deficient, in contradiction to its construction.

We can further assume without loss of generality that the first nonzero entry of every column of $\B$ is $1$ (otherwise we may just rescale each column), and that these nonzero entries are in the first columns (otherwise we may just permute the columns accordingly).  We will also let \phantomsection\label{BtwoTildeDef}$\BtwoTilde$ denote all but the first row of $\Btwo$. Thus, our matrix is organized as
\begin{align}
\small
\label{AconstructionEq}
\left[\begin{array}{c|c}
\Adp & \ai
\end{array}\right] =
\begin{matrix}
\begin{matrix} \\\vspace{.1cm} \\ \\ \\ \\  \end{matrix} \\
\left. \begin{matrix} \\ \\ \vspace{.3cm} \\ \\ \end{matrix}  \Btwo\right\{ \\
\begin{matrix} \\ \\ \\ \\ \end{matrix}
\end{matrix}
\left[\begin{array}{cc|c}
\multicolumn{2}{c|}{\hspace{.75cm} \Scale[1.5]{\bs{0}} \hspace{.75cm}} & \bs{0}\\ \hline
&&\\ \multicolumn{2}{c|}{\Scale[1.5]{\C}} & \hatai \\&&\\ \hline
\multicolumn{1}{c|}{\hspace{.3cm} \Scale[1]{\one} \hspace{.3cm}} & \hspace{.1cm} \Scale[1]{\bs{0}} & 0 \\ \hline
&&\\ \multicolumn{2}{c|}{\Scale[1.5]{\BtwoTilde}} & \bs{0} \\&&\\ \hline
&&\multirow{2}{*}{\Scale[1]{\bs{0}}} \\ 
 \multicolumn{2}{c|}{\Scale[1.5]{\Bone}} & \\ \cline{3-3}&&1\\
\end{array} \right]
\begin{matrix}
\left. \begin{matrix} \\ \end{matrix} \right\} \d-\m \hspace{.5cm} \\
\left. \begin{matrix} \\ \\ \vspace{.1cm} \\ \end{matrix} \right\} \r \hspace{1.1cm} \\
\left. \begin{matrix} \\ \end{matrix} \right\} 1 \hspace{1.1cm} \\
\left. \begin{matrix} \\ \\ \vspace{.2cm} \\ \end{matrix} \right\} \begin{array}{l}\m-\n \\ -\r-1 \\ \geq 0\end{array} \\
\left. \begin{matrix} \\ \\ \end{matrix} \right\} \n-1 \hspace{.5cm} \\
\left. \begin{matrix} \\ \end{matrix} \right\} 1. \hspace{0.9cm}
\end{matrix}
\normalsize
\end{align}

Now \eqref{aLdOnAEq} implies $\Bone \bb = [ \hspace{.1cm} \bs{0} \hspace{.1cm} | \hspace{.1cm} 1 \hspace{.1cm}]^\T$, and since $\Bone$ is full rank, we may write
\begin{align*}
\bb \ = \ \BoneInv \left[\begin{matrix} \bs{0} \\ 1 \end{matrix}\right],
\end{align*}
i.e., $\bb$ is the the last column of the inverse of $\Bone$, which is a rational function in the elements of $\Bone$.

Next, let us look back at \eqref{aLdOnAEq}.  If $\n<\m-\r$, then using the additional row $[ \hspace{.1cm} \one \hspace{.1cm} | \hspace{.1cm} \bs{0} \hspace{.1cm} ]$ of  \eqref{AconstructionEq} (which does not appear if $\m=\n+\r$) we obtain $[ \hspace{.1cm} \one \hspace{.1cm} | \hspace{.1cm} \bs{0} \hspace{.1cm}] \bb  = 0$.   Recall that all the entries of $\bb$ are nonzero.  
Thus, the last equation defines the following nonzero rational function in the elements of $\Bone$:
\begin{align}
\label{polyEq}
\left[\begin{array}{c|r} \one & \bs{0} \end{array}\right] \BoneInv\left[\begin{matrix} \bs{0} \\ 1 \end{matrix}\right] \ = \ 0 .
\end{align}
Equivalently, \eqref{polyEq} is a polynomial equation in the elements of $\Bone$, which we will denote as $\f(\Bone)=0$.

Next note that for \ae\ $\sstar$, we can write \phantomsection\label{AAstarDef}$\sstar=\ker \AAstar{}^\T$ for a unique $\AA^\star \in \R^{\fix{\d} \times (\fix{\d}-\fix{\r})}$ in column echelon form\footnote{Certain $\sstar$ may not admit this representation, e.g., if $\sstar$ is orthogonal to certain canonical coordinates, which, as discussed in \aEntriesLem, is not the case for almost every $\sstar$ in $\Gr(\r,\R^{\fix{\d}})$.}:
\begin{align}
\label{AAstarEq}
\AAstar \ = \ \left[ \begin{array}{c}
\\ \Scale[2]{\I} \\ \\ \hline
\hspace{.4cm} \Scale[1.5]{\bs{\DDstar}} \hspace{.2cm} \\
\end{array}\right]
\begin{matrix}
\left. \begin{matrix} \\ \\ \\ \end{matrix} \right\} \d-\r \\
\left. \begin{matrix} \\ \end{matrix} \right\} \r \hspace{.5cm}.
\end{matrix}
\end{align}
On the other hand every $\DDstar \in \R^{\fix{\r} \times (\fix{\d}-\fix{\r})}$ defines a unique $\r$-dimensional subspace of $\R^{\fix{\d}}$, via \eqref{AAstarEq}.  Thus, we have a bijection between $\R^{\fix{\r} \times (\fix{\d}-\fix{\r})}$ and a dense open subset of $\Gr(\r,\R^{\fix{\d}})$.

Since the columns of $\Adp$ must be linear combinations of the columns of $\AAstar$, the elements of $\Bone$ are linear functions in the entries of $\DDstar$.  Therefore, we can express $\f(\Bone)$ as a nonzero polynomial function $g$ in the entries of $\DDstar$ and rewrite \eqref{polyEq} as $g(\DDstar)=0$.  But we know that $g(\DDstar) \neq 0$ for almost every $\DDstar \in \R^{\fix{\r} \times (\fix{\d}-\fix{\r})}$, and hence for almost every $\sstar \in \Gr(\r,\R^{\fix{\d}})$. We conclude that almost every subspace in $\Gr(\r,\R^{\fix{\d}})$ will not satisfiy \eqref{polyEq}, and thus $\n = \m - \r$.
\end{proof}

We are now ready to present the proofs of \independenceLem\ and \identifiabilityThm.

\begin{proof}(\independenceLem)
\begin{itemize}
\item[($\Rightarrow$)]
Suppose $\A$ is \mld.  By \basisLem, $\nOf(\A)=\mOf(\A)-\r+1>\mOf(\A)-\r$, and we have the first implication.

\item[($\Leftarrow$)]
Suppose there exists an $\A$ with $\nOf(\A)>\mOf(\A)-\r$.  By \liLem, $\nOf(\A)>\liOf(\A)$, which implies the columns in $\A$, and hence $\AA$, are linearly dependent.
\end{itemize}
\end{proof}


\begin{proof}(\identifiabilityThm)
\aEntriesLem\ shows that for \ae\ $\sstar$, the $(\j,\i)^{th}$ entry of $\AA$ is nonzero if and only if the $(\j,\i)^{th}$ entry of $\OO$ is nonzero.

\begin{itemize}
\item[($\Rightarrow$)]
Suppose there exists an $\O$ such that $\mOf(\O) < \nOf(\O)+\r$. Then $\mOf(\A) < \nOf(\A)+\r$ for some $\A$.  \independenceLem\ implies that the columns of $\A$, and hence $\AA$, are linearly dependent.  This implies  $\dim \ker\AA{}^\T > \r$.

\item[($\Leftarrow$)]
Suppose every $\O$ satisfies $\mOf(\O) \geq \nOf(\O)+\r$. Then $\mOf(\A) \geq \nOf(\A)+\r$ for every $\A$, including $\AA$. Therefore, by \independenceLem, the $\d-\r$ columns in $\AA$ are linearly independent, hence $\dim \ker \AA{}^\T=\r$.
\end{itemize}
\end{proof}

\vspace{.3cm}
\section{Implications for low-rank matrix completion}
\label{LRMCSec}
\vspace{.3cm}
Subspace identifiability is closely related to the \phantomsection\label{LRMCDef}low-rank matrix completion (\LRMC) problem \cite{candes}: given a subset of entries in a rank-$\r$ matrix, exactly recover {\em all} of the missing entries.  This requires, implicitly, idenficiation of the subspace spanned by the complete columns of the matrix.   We use this section to present the implications of our results for \LRMC.

Let \phantomsection\label{XDef}$\X$ be a $\d \times \N$, rank-$\r$ matrix and assume that
\begin{itemize}
\phantomsection\label{AthreeAssDef}
\item[\AthreeAss]
The columns of $\X$ are drawn independently according to \phantomsection\label{nuuDef}$\nuu$, an absolutely continuous distribution with respect to the Lebesgue measure on $\sstar$.
\end{itemize}
Let \phantomsection\label{XODef}$\XO$ be the incomplete version of $\X$, observed only in the nonzero positions of $\OO$.

\subsection*{Necessary and sufficient conditions for \LRMC}

To relate the \LRMC\ problem to our main results, define \phantomsection\label{NtildeDef}$\Ntilde$ as the number of {\em distinct} columns (sampling patterns) in $\OO$, and let \phantomsection\label{OTildeDef}$\OTilde$ denote a $\d \times \Ntilde$ matrix composed of these columns.

\begin{myCorollary}
\label{LRMCnecCor}
If $\OTilde$ does not contain a $\d \times (\d-\r)$ submatrix satisfying the conditions of \identifiabilityThm, then $\X$ cannot be uniquely recovered from $\XO$.
\end{myCorollary}

Since $\X$ is rank-$\r$, a column with fewer than $\r$ observed entries cannot be completed (in general).  We will thus assume without loss of generality the following relaxation of \AoneAss:
\begin{itemize}
\phantomsection\label{AonepAssDef}
\item[\AonepAss]
$\OO$ has at least $\r$ nonzero entries per column.
\end{itemize}

\begin{myCorollary}
\label{LRMCsuffCor}
Let \AonepAss\ and \AthreeAss\ hold.  Suppose $\OTilde$ contains a $\d \times (\d-\r)$ submatrix satisfying the conditions of \identifiabilityThm, and that for every column $\oi$ in this submatrix, at least $\r$ columns in $\XO$ are observed at the nonzero locations of $\oi$.  Then for \ae\ $\sstar$, and almost surely with respect to $\nuu$, $\X$ can be uniquely recovered from $\XO$.
\end{myCorollary}

Proofs of these results are given in the \LRMCnecApx.  The intuition behind \LRMCnecCor\ is simply that identifying a subspace from its projections onto sets of canonical coordinates is {\em easier } than \LRMC, and so the necessary condition of  \identifiabilityThm \ is also necessary for \LRMC.  \LRMCsuffCor\ follows from the fact that $\sstar$ (or its projections) can be determined from $\r$ or more observations drawn from $\nuu$.


\subsection*{Validating \LRMC}
Under certain assumptions on the subset of observed entries (e.g., random sampling) and $\sstar$ (e.g., incoherence), existing methods, for example nuclear norm minimization \cite{candes}, succeed {\em with high probability} in completing the matrix exactly and thus identifying $\sstar$.  These assumptions are sufficient, but not necessary, and are sometimes unverifiable or unjustified in practice.
Therefore, the result of an \LRMC\ algorithm can be suspect. Simply finding a low-rank matrix that agrees with the observed data does not guarantee that it is the correct completion.  It is possible that there exist other $\r$-dimensional subspaces different from $\sstar$ that agree with the observed entries.

\begin{myExample}
Suppose we run an \LRMC\ algorithm on a matrix observed on the support of $\OO$, with $\OO$ and $\sstar$ as in \introEg \ in Section~\ref{modelSec}.  Suppose that the algorithm produces a completion with columns from   $\s={\rm span}[1 \ \ 2 \ \ 3 \ \ 5 \ \ 5]^\T$ instead of $\sstar$.  It is clear that the residual of the projection of any vector from $\sstaroi$ onto $\soi$ will be zero, despite the fact that $\s \neq \sstar$.
\end{myExample}

In other words, if the residuals are nonzero, we can discard an incorrect solution, but if the residuals are zero, we cannot validate whether our solution is correct or not.

\vspace{.3cm}
\begingroup
\leftskip1.5em
\rightskip\leftskip
\noindent
{\em \fitsCor, below, allows one to drop the sampling and incoherence assumptions, and validate the result of {\em any} \LRMC\ algorithm deterministically.}
\par
\endgroup
\vspace{.3cm}

Let \phantomsection\label{xDef}$\x_i$ denote the $\i^{th}$ column of $\X$, and \phantomsection\label{xoiDef}$\xoi$ be the restriction of $\x_i$ to the nonzero coordinates of $\oi$. We say that a subspace $\s$ \phantomsection\label{fitsDef}\fits\ $\XO$ if $\xoi \in \soi$ for every $\i$.

\begin{myCorollary}
\label{fitsCor}
Let \AthreeAss\ hold, and suppose $\XO$ contains two disjoint sets of columns, \phantomsection\label{XOoneDef}$\XOone$ and \phantomsection\label{XOtwoDef}$\XOtwo$, such that \phantomsection\label{OOtwoDef}$\OOtwo$ is a $\d \times (\d-\r)$ matrix satisfying the conditions of \identifiabilityThm.  Let $\s$ be the subspace spanned by the columns of a completion of  $\XOone$. Then for \ae\ $\sstar$, and almost surely with respect to $\nuu$, $\s$ \fits\ $\XOtwo$ if and only if $\s=\sstar$.
\end{myCorollary}

The proof of \fitsCor\ is given in the \fitsApx.  In words, \fitsCor\ states that if one runs an \LRMC\ algorithm on $\XOone$, then the uniqueness and correctness of the resulting low-rank completion can be verified by testing whether it agrees with the {\em validation} set $\XOtwo$.

\begin{myExample}
Consider a $1000 \times 2000$ matrix $\XO$ with $\r=30$ and ideal incoherence.  In this case, the best sufficient conditions for \LRMC\ that we are aware of \cite{recht} require that all entries are observed. Simulations show that alternating minimization \cite{jain} can exactly complete such matrices when fewer than half of the entries are observed, and only using half of the columns.  While previous theory for matrix completion gives no guarantees in scenarios like this, our new results do.

To see this, split $\XO$ into two $1000 \times 1000$ submatrices $\XOone$ and $\XOtwo$. Use nuclear norm, alternating minimization, or any \LRMC\ method, to find a completion of $\XOone$.  \probabilityThm\ can be used to show that the sampling of $\XOtwo$ will satisfy the conditions of \identifiabilityThm\ \whp. even when only half the entries are observed randomly.  We can then use \fitsCor\ to show that if $\XOtwo$ is consistent with the completion of $\XOone$, then the completion is unique and correct.
\end{myExample}

\subsection*{Remarks}
Observe that the necessary and sufficient conditions in Corollaries \ref{LRMCnecCor} and \ref{LRMCsuffCor} and the validation in \fitsCor\ do not require the incoherence assumptions typically needed in \LRMC\ results in order to guarantee correctness and uniqueness.

Another advantage of results above is that they work for matrices of any rank, while standard \LRMC\ results only hold for ranks significantly smaller than the dimension $\d$.

Finally, the results above hold with probability $1$, as opposed to standard \LRMC\ statements, that hold \whp.  On the other hand, verifying whether $\OOtwo$ meets the conditions of \identifiabilityThm\ may be difficult.  Nevertheless, if the entries in our data matrix are sampled randomly with rates comparable to standard conditions in \LRMC, we know by \probabilityThm\ that \whp. $\OOtwo$ will satisfy such conditions.

\section{Graphical interpretation of the problem}
\label{graphSec}
The problem of \LRMC\ has also been studied from the graph theory perspective.  For example, it has been shown that graph connectivity is a necessary condition for completion \cite{kiraly}.  Being subspace identifiability so tightly related to \LRMC, it comes as no surprise that there also exist graph conditions for subspace identifiability.  In this section we draw some connections between subspace identifiability and graph theory that give insight on the conditions in \identifiabilityThm.  We use this interpretation to show that graph connectivity is a necessary yet insufficient condition for subspace identification.

Define \phantomsection\label{GODef}$\GOO$ as the bipartite graph with disjoint sets of {\em row} and {\em column} vertices, where there is an edge between row vertex $\j$ and column vertex $\i$ if the $(\j,\i)^{th}$ entry of $\OO$ is nonzero.

\begin{myExample}
\label{graphEg}
With $\d=5$, $\r=1$ and
\begin{center}
	\begin{tikzpicture}
		\node [title] (c) at (-3,-1) {
			$\OO = \left[\begin{matrix}
			1 & 0 & 0 & 0 \\
			1 & 1 & 0 & 0 \\
			0 & 1 & 1 & 0 \\
			0 & 0 & 1 & 1 \\
			0 & 0 & 0 & 1 \\
                    \end{matrix}\right]$};
                    
                 \node [title] (c) at (-1,-1) {$\Rightarrow$};
		
		\node [title] (c) at (1,-2.25) {$\GOO$};
		\node [title] (c) at (2,.5) {Columns};
		\node [title] (r) at (0,.5) {Rows};
		\node [label] (c1) at (2,-0.25){$1$}; \node [vertex] (c1) at (2,-0.25){};
		\node [label] (c2) at (2,-.75) {$2$}; \node [vertex] (c2) at (2,-.75) {};
		\node [label] (c3) at (2,-1.25) {$3$}; \node [vertex] (c3) at (2,-1.25) {};
		\node [label] (c4) at (2,-1.75) {$4$}; \node [vertex] (c4) at (2,-1.75) {};
		
		\node [label] (r1) at (0,0) {$1$}; \node [vertex] (r1) at (0,0) {};
		\node [label] (r2) at (0,-.5) {$2$}; \node [vertex] (r2) at (0,-.5) {};
		\node [label] (r3) at (0,-1) {$3$}; \node [vertex] (r3) at (0,-1) {};
		\node [label] (r4) at (0,-1.5) {$4$}; \node [vertex] (r4) at (0,-1.5) {};
		\node [label] (r5) at (0,-2) {$5$}; \node [vertex] (r5) at (0,-2) {};
		\draw [font=\scriptstyle]
				(c1) edge (r1)
				(c1) edge (r2)
				(c2) edge (r2)
				(c2) edge (r3)
				(c3) edge (r3)
				(c3) edge (r4)
				(c4) edge (r4)
				(c4) edge (r5);
	\end{tikzpicture}.
\end{center}
\end{myExample}

Recall that the \phantomsection\label{neighborhoodDef}{\em \neighborhood} of a set of vertices is the collection of all their adjacent vertices.

\vspace{.3cm}
\begingroup
\leftskip1.5em
\rightskip\leftskip
\noindent
{\em The graph theoretic interpretation of the condition on $\OO$ in \identifiabilityThm\ is that every set of $\nOf$ column vertices in $\GOO$ must have a \neighborhood\ of at least $\nOf+\r$ row vertices.}
\par
\endgroup
\vspace{.3cm}

\begin{myExample}
One may verify that every set of $\nOf$ column vertices in $\GOO$ from \graphEg\ has a \neighborhood\ of at least $\nOf+\r$ row vertices.
\pagebreak
On the other hand, if we consider $\OO$ as in \introEg, the \neighborhood\ of the column vertices $\{1,2,3\}$ in $\GOO$ contains fewer than $\nOf+\r$ row vertices:
\begin{center}
	\begin{tikzpicture}
		\node [title] (c) at (1,-2.25) {$\GOO$};
		\node [title] (c) at (2,.5) {Columns};
		\node [title] (r) at (0,.5) {Rows};
		\node [label] (c1) at (2,-0.25){$1$}; \node [vertex, line width=1.25pt] (c1) at (2,-0.25){};
		\node [label] (c2) at (2,-.75) {$2$}; \node [vertex, line width=1.25pt] (c2) at (2,-.75) {};
		\node [label] (c3) at (2,-1.25) {$3$}; \node [vertex, line width=1.25pt] (c3) at (2,-1.25) {};
		\node [label] (c4) at (2,-1.75) {$4$}; \node [vertex] (c4) at (2,-1.75) {};
		
		\node [label] (r1) at (0,0) {$1$}; \node [vertex, line width=1.25pt] (r1) at (0,0) {};
		\node [label] (r2) at (0,-.5) {$2$}; \node [vertex, line width=1.25pt] (r2) at (0,-.5) {};
		\node [label] (r3) at (0,-1) {$3$}; \node [vertex, line width=1.25pt] (r3) at (0,-1) {};
		\node [label] (r4) at (0,-1.5) {$4$}; \node [vertex] (r4) at (0,-1.5) {};
		\node [label] (r5) at (0,-2) {$5$}; \node [vertex] (r5) at (0,-2) {};
		\draw [font=\scriptstyle]
				(c1) edge[line width=.75pt] (r1)
				(c1) edge[line width=.75pt] (r2)
				(c2) edge[line width=.75pt] (r2)
				(c2) edge[line width=.75pt] (r3)
				(c3) edge[line width=.75pt] (r1)
				(c3) edge[line width=.75pt] (r3)
				(c4) edge (r4)
				(c4) edge (r5);
	\end{tikzpicture}.
\end{center}
\end{myExample}

With this interpretation of \identifiabilityThm, we can extend terms and results from graph theory to our context.  One example is the next corollary, which states that $\r$-row-connectivity is a necessary but insufficient condition for subspace identifiability.

We say $\GOO$ is {\em $\r$-row-connected} if $\GOO$ remains a connected graph after removing any set of $\r-1$ row vertices and all their adjacent edges.

\begin{myCorollary}
\label{rConnectedCor}
For \ae\ $\sstar$, $|\SS(\sstar,\OO)|>1$ if $\GOO$ is not $\r$-row-connected.  The converse is only true for $\r=1$.
\end{myCorollary}

\rConnectedCor\ is proved in the \rConnectedApx.

\section{Conclusions}
\label{conclusionsSec}
In this paper we determined when and only when can one identify a subspace from its projections onto subsets of the canonical coordinates.  We show that the conditions for identifiability hold \whp. under standard random sampling schemes, and that when these conditions are met, identifying the subspace becomes a trivial task.

This gives new necessary and sufficient conditions for \LRMC, and allows one to verify whether the result of {\em any} \LRMC\ algorithm is unique and correct without prior incoherence or sampling assumptions.


\newpage
\section*{Appendix}
\label{appendix}

\subsection*{Generalization of Our Results}
\label{generalizationApx}
Since the restriction of $\sstar$ onto $\L \leq \r$ coordinates will be $\R^{\fix{\L}}$ (in general), such a projection will provide no information specific to $\sstar$.  We will thus assume without loss of generality that:
\begin{itemize}
\phantomsection\label{AonedpAssDef}
\item[\AonedpAss]
$\O$ has {\em at least} $\r+1$ nonzero entries per column.
\end{itemize}

Under \AonedpAss, a column with $\L$ observed entries restricts $\SS(\sstar,\OO)$ just as $\L-\r$ columns under \AoneAss.  Thus in general, if there are columns in $\OO$ with more than $\r+1$ nonzero entries, we can {\em split} them to obtain an {\em expanded} matrix $\OSplit$ (defined below), with exactly $\r+1$ nonzero entries per column, and use \identifiabilityThm\ directly on this expanded matrix.

More precisely, let \phantomsection\label{kiDef}$\ki{1},\dots,\ki{\ell_i}$ denote the indices of the \phantomsection\label{LiDef}$\Li$ nonzero entries in the $\i^{th}$ column of $\OO$.  Define \phantomsection\label{OiDef}$\Oi{i}$ as the $\d \times (\Li-\r)$ matrix, whose \phantomsection\label{jDef}$\j^{th}$ column has the value $1$ in rows $\ki{1},\dots,\ki{r}$, $\ki{r+j}$, and zeros elsewhere.  For example, if $\ki{1}=1,\dots,\ki{\ell_i}=\Li$, then
\begin{align*}
\Oi{i} \ = \ 
\left[ \begin{matrix} \\ \\ \\ \\ \\ \\ \\ \end{matrix} \right.
\underbrace{
\begin{matrix}
\multirow{2}{*}{\hspace{.3cm}$\Scale[1.5]{\one}$\hspace{.35cm}} \\ \\ \hline
\multirow{3}{*}{$\Scale[1.5]{\I}$} \\ \\ \\ \hline
\multirow{2}{*}{$\Scale[1.5]{\bs{0}}$} \\ \\
\end{matrix}}_{\fix{\Li}-\fix{\r}}
\left] \begin{matrix}
\left. \begin{matrix} \\ \\ \end{matrix} \right\} \r \hspace{.7cm} \\
\left. \begin{matrix} \\ \\ \\ \end{matrix} \right\} \Li-\r \hspace{.1cm} \\
\left. \begin{matrix} \\ \\ \end{matrix} \right\} \d-\Li,
\end{matrix} \right.
\end{align*}
where $\one$ denotes a block of all $1$'s and $\I$ the identity matrix.  Finally, define \phantomsection\label{OSplitDef}$\OSplit \ := \ [ \Oi{1} \ \cdots \ \Oi{N} ]$.

The following is a generalization of \identifiabilityThm\ to an arbitrarily number of projections and an arbitrary number of canonical coordinates involved in each projection.  It states that $\sstar$ will be the only subspace in $\SS(\sstar,\OO)$ if and only if there is a matrix \phantomsection\label{OHatDef}$\OHat$, formed with $\d-\r$ columns of $\OSplit$, that satisfies the conditions of \identifiabilityThm.

\begin{framed}
\begin{myTheorem}
Let \AonedpAss\ hold.  For \hyperref[aeDef]{almost every} $\sstar$, $\sstar$ is the only subspace in $\SS(\sstar,\OO)$ if and only if there is a matrix $\OHat$, formed with $\d-\r$ columns of $\OSplit$, such that every matrix $\O$ formed with a subset of the columns in $\OHat$ satisfies \eqref{identifiabilityEq}.
\end{myTheorem}
\end{framed}

\begin{proof}
It suffices to show that $\SS(\sstar,\OO) = \SS(\sstar,\OSplit)$.  Let \phantomsection\label{oijDef}$\oij{j}$ denote the $\j^{th}$ column of $\Oi{i}$.
\begin{itemize}
\item[($\subset$)]
Let $\s \in \SS(\sstar,\OO)$.  By definition, $\soi=\sstaroi$, which trivially implies $\{\s_{\fix{\oij{j}}}=S^\star_{\fix{\oij{j}}}\}_{\fix{\j}=1}^{\fix{\Li}-\fix{\r}}$.  Since this is true for every $\i$, we conclude $\s \in \SS(\sstar,\OSplit)$.

\item[($\supset$)]
Let $\s \in \SS(\sstar,\OSplit)$.  By definition, $\{\s_{\fix{\oij{j}}}=S^\star_{\fix{\oij{j}}}\}_{\fix{\j}=1}^{\fix{\Li}-\fix{\r}}$.  Notice that $\Oi{i}$ satisfies the conditions of \identifiabilityThm\ restricted to the nonzero rows in $\oi$, which implies $\soi=\sstaroi$.  Since this is true for every $\i$, we conclude $\s \in \SS(\sstar,\OO)$.
\end{itemize}
\vspace{-.5cm}
\end{proof}

\subsection*{Proof of \probabilityThm}
\label{probabilityApx}
Let \phantomsection\label{EDef}$\E$ be the event that $\OO$ fails to satisfy the conditions of \identifiabilityThm.
It is easy to see that this may only occur if there is a matrix formed with $\nOf$ columns from $\OO$ that has all its nonzero entries in the same $\nOf+\r-1$ rows.  Let \phantomsection\label{EnDef}$\En$ denote the event that the matrix formed with the first $\nOf$ columns from $\OO$ has all its nonzero entries in the first $\nOf+\r-1$ rows.  Then
\begin{align}
\label{badSetProbEq}
\P\left( \E \right) 
&\leq \sum_{\fix{\nOf}=1}^{\fix{\d}-\fix{\r}} {\d-\r \choose \nOf} {\d \choose \nOf+\r-1} \P\left( \En \right)
\end{align}
If each column of $\OO$ contains at least $\L$ nonzero entries, distributed uniformly and independently at random with $\L$ as in \eqref{kEq}, it is easy to see that $\P(\En)=0$ for $\nOf \leq \L-\r$, and for $\L-\r < \nOf \leq \d-\r$,
\begin{align*}
\P(\En) \leq \left(\frac{{\fix{\nOf}+\fix{\r}-1 \choose \fix{\L}{}}} {{\fix{\d}{} \choose \fix{\L}{}}}\right)^{\fix{\nOf}}
&< \left(\frac{\nOf+\r-1}{\d}\right)^{\fix{\L}\fix{\nOf}}.
\end{align*}
Since ${\fix{\d}-\fix{\r} \choose \fix{\nOf}{}} < {\fix{\d}{} \choose \fix{\nOf}+\fix{\r}-1}$,
continuing with \eqref{badSetProbEq} we obtain:
\begin{align}
\P \left( \E \right) &< \sum_{\fix{\nOf}=\fix{\L}-\fix{\r}+1}^{\fix{\d}-\fix{\r}} {\d \choose \nOf+\r-1}^2 \left(\frac{\nOf+\r-1}{\d}\right)^{\fix{\L}\fix{\nOf}} \nonumber \\
&< \sum_{\fix{\nOf}=\fix{\L}}^{\frac{\fix{\d}{}}{2}} {\d \choose \nOf}^2 \left(\frac{\nOf}{\d}\right)^{\fix{\L}(\fix{\nOf}-\fix{\r}+1)} \label{firstPartProbEq} \\
&+ \sum_{\fix{\nOf}=1}^{\frac{\fix{\d}{}}{2}} {\d \choose \d-\nOf}^2 \left(\frac{\d-\nOf}{\d}\right)^{\fix{\L}(\fix{\d}-\fix{\nOf}-\fix{\r}+1)}. \label{secondPartProbEq}
\end{align}
For the terms in \eqref{firstPartProbEq}, write
\begin{align}
\label{firstTermEq}
{\d \choose \nOf}^2 \left(\frac{\nOf}{\d}\right)^{\fix{\L}(\fix{\nOf}-\fix{\r}+1)} &\leq \left({\frac{\d e}{\nOf}}\right)^{2\fix{\nOf}} \left(\frac{\nOf}{\d}\right)^{\fix{\L}(\fix{\nOf}-\fix{\r}+1)}.
\end{align}
Since $\n \geq \L \geq 2\r$,
\begin{align}
\label{someEq}
\eqref{firstTermEq} &< \left({\frac{\d e}{\nOf}}\right)^{2\fix{\nOf}} \left(\frac{\nOf}{\d}\right)^{\fix{\L} \frac{\fix{\nOf}{}}{2}} 
= e^{2\fix{\nOf}} \left(\frac{\nOf}{\d}\right)^{(\frac{\fix{\L}{}}{2}-2)\fix{\nOf}},
\end{align}
and since $\nOf \leq \frac{\fix{\d}{}}{2}$,
\begin{align}
\label{firstBoundEq}
\eqref{someEq}
&\leq e^{2\fix{\nOf}} \left(\frac{1}{2}\right)^{(\frac{\fix{\L}{}}{2}-2)\fix{\nOf}} 
= \left( e^2 \cdot 2^{-\frac{\fix{\L}{}}{2}+2} \right)^{\fix{\nOf}} 
< \frac{\eps}{\d},
\end{align}
where the last step follows because $\L > 2\log_2(\frac{\fix{\d} e^2}{\fix{\eps}{}})+4$.

For the terms in \eqref{secondPartProbEq}, write
\begin{align}
\small
\label{secondTermEq}
{\d \choose \d-\nOf}^2 \left(\frac{\d-\nOf}{\d}\right)^{\fix{\L}(\fix{\d}-\fix{\nOf}-\fix{\r}+1)} &\leq \left(\frac{\d e}{\nOf}\right)^{2\fix{\nOf}} \left(\frac{\d-\nOf}{\d}\right)^{\fix{\L}(\fix{\d}-\fix{\nOf}-\fix{\r}+1)}.
\normalsize
\end{align}
In this case, since $1 \leq \nOf \leq \frac{\fix{\d}{}}{2}$ and $\r \leq \frac{\fix{\d}{}}{6}$, we have
\begin{align*}
\eqref{secondTermEq}
< (\d e)^{2\fix{\nOf}} \left(\frac{\d-\nOf}{\d}\right)^{\fix{\L} \frac{\fix{\d}{}}{3}} 
&= (\d e)^{2\fix{\nOf}} \left[\left(1-\frac{\nOf}{\d}\right)^{\fix{\d}} \right]^{\frac{\fix{\L}{}}{3}} \\
&\leq (\d e)^{2\fix{\nOf}} \left[e^{-\fix{\nOf}}\right]^{\frac{\fix{\L}{}}{3}},
\end{align*}
which we may rewrite as
\begin{align}
\label{secondBoundEq}
\left(e^{2\log\fix{\d}}\right)^{\fix{\n}} \left(e^2\right)^{\fix{\n}} \left(e^{-\frac{\fix{\L}{}}{3}}\right)^{\fix{\n}} = \left(e^{2\log\fix{\d} + 2 - \frac{\fix{\L}{}}{3} }\right)^{\fix{\n}} < \frac{\eps}{\d},
\end{align}
where the last step follows because $\L > 3 \log(\frac{\fix{\d}{}}{\fix{\eps}{}}) + 6 \log \d + 6$.
Substituting \eqref{firstBoundEq} and \eqref{secondBoundEq} in \eqref{firstPartProbEq} and \eqref{secondPartProbEq}, we have that $\P(\E)<\eps$, as desired. \hfill $\square$



\subsection*{Proof of \LRMCnecCor}
\label{LRMCnecApx}
A subspace satisfying $\soi=\sstaroi$ will \fit\ all the columns of $\XO$ observed in the nonzero positions of $\oi$.  Therefore, any subspace that satisfies $\soi=\sstaroi$ for every $\oi$ in $\OTilde$ will \fit\ all the columns in $\XO$.  If $\OTilde$ does not satisfy the conditions of \identifiabilityThm, there will exist multiple subspaces that \fit\ $\XO$, whence $\X$ cannot be uniquely recovered from $\XO$. \hfill $\square$

\subsection*{Proof of \LRMCsuffCor}
\label{LRMCsuffApx}
Suppose there are at least $\r$ columns in $\XO$ observed in the nonzero positions of $\oi$.  Then almost surely with respect to $\nuu$, the restrictions of such columns form a basis for $\sstaroi$.  Therefore, any subspace $\s$ that \fits\ such columns must satisfy $\soi = \sstaroi$.  If this is true for every $\oi$ in a $\d \times (\d-\r)$ submatrix of $\OTilde$, then any subspace that \fits\ $\XO$ must satisfy $\soi=\sstaroi$ for every $\oi$ in this submatrix.

There will be only one subspace that satisfies this condition if this submatrix satisfies the conditions in \identifiabilityThm.  Finally, observe that under \AonepAss, the condition that $\X$ can be uniquely recovered from $\XO$ is equivalent to saying that $\sstar$ is the only $\r$-dimensional subspace that \fits\ $\XO$. \hfill $\square$

\subsection*{Proof of \fitsCor}
\label{fitsApx}
\begin{itemize}
\item[($\Leftarrow$)]
$\xoi \in \sstaroi$ by assumption, so if $\s=\sstar$, it is trivially true that $\xoi \in \soi$.

\item[($\Rightarrow$)]
Use $\i=1,\dots,(\d-\r)$ to index the columns in $\XOtwo$.   Since $\s$ \fits\ $\XOtwo$, by definition $\xoi \in \soi$.  On the other hand, $\xoi \in \sstaroi$ by assumption, which implies that for every $\i$, $\xoi$ lies in the intersection of $\soi$ and $\sstaroi$.  Recall that $\x_i$ is sampled independently according to $\nuu$, an absolutely continuous distribution with respect to the Lebesgue measure on $\sstar$.  Since $\dim \soi \leq \r$, and for \ae\ $\sstar$, $\dim \sstaroi=\r$, the event
\begin{align*}
\bigcap_{\fix{\i}=1}^{\fix{\d}-\fix{\r}} \left\{\xoi \in \soi \cap \sstaroi \right\}
\end{align*}
will (almost surely with respect to $\nuu$) only happen if $\soi = \sstaroi$ $\forall$ $\i$, that is, if $\s \in \SS(\sstar,\OOtwo)$.  Since $\OOtwo$ satisfies the conditions of \identifiabilityThm, $\sstar$ is the only subspace in $\SS(\sstar,\OOtwo)$.  This implies $\s=\sstar$, which concludes the proof. \hfill $\square$
\end{itemize}

\subsection*{Proof of \rConnectedCor}
\label{rConnectedApx}
\begin{itemize}
\item[($\Leftarrow$)]
Suppose $\GOO$ is not $\r$-row-connected.  This means there exists a set \phantomsection\label{JDef}$\J$ of $\r-1$ row vertices such that if removed with their respective edges, $\GOO$ becomes a disconnected graph.

Let \phantomsection\label{JpDef}$\Jp$, \phantomsection\label{JdpDef}$\Jdp$ and $\J$ be a partition of the row vertices in $\GOO$ such that $\Jp$ and $\Jdp$ become disconnected when $\J$ is removed.

Similarly, let \phantomsection\label{OpDef}$\Op$ and \phantomsection\label{OdpDef}$\Odp$ be a partition of the columns in $\OO$ such that the column vertices corresponding to $\Op$ are disconnected from the row vertices in $\Jdp$, and the column vertices corresponding to $\Odp$ are disconnected from the row vertices in $\Jp$.

\begin{center}
	\begin{tikzpicture}
		\node [title] (c) at (1,-3.5) {$\GOO$};
		\node [title] (c) at (2,.5) {Columns};
		\node [title] (r) at (0,.5) {Rows};
		
		\node [vertex] (c1) at (2,-.5) {};
		\node [label] (cdots1) at (2,-0.875){$\vdots$}; \node [label] (cdots1) at (2,-0.875) {};
		\node [vertex] (c2) at (2,-1.25) {};
		\node [title](CwtO) at (2.58,-.8725){$\left\} \begin{matrix} \\ \\ \\ \end{matrix} \Op \right.$};
		
		\node [vertex] (c3) at (2,-2) {};
		\node [label] (cdots2) at (2,-2.375){$\vdots$}; \node [label] (cdots2) at (2,-2.375) {};
		\node [vertex] (c4) at (2,-2.75) {};
		\node [title](CwtO) at (2.58,-2.3725){$\left\} \begin{matrix} \\ \\ \\ \end{matrix} \Odp \right.$};
		
		\node [vertex] (r1) at (0,0) {};
		\node [label] (rdots1) at (0,-0.375){$\vdots$}; \node [label] (rdots1) at (0,-0.35){};
		\node [vertex] (r2) at (0,-.75) {};
		\node [title](CwtO) at (-.55,-0.375){$\left. \begin{matrix} \\ \\ \\ \end{matrix} \Jp \right\{$};
		
		\node [vertex] (r3) at (0,-1.25) {};
		\node [label] (rdots2) at (0,-1.625){$\vdots$}; \node [label] (rdots2) at (0,-1.65){};
		\node [vertex] (r4) at (0,-2) {};
		\node [title](CwtO) at (-.5,-1.625){$\left. \begin{matrix} \\ \\ \\ \end{matrix} \J \right\{$};
		
		\node [vertex] (r5) at (0,-2.5) {};
		\node [label] (rdots3) at (0,-2.875){$\vdots$}; \node [label] (rdots3) at (0,-2.875){};
		\node [vertex] (r6) at (0,-3.25) {};
		\node [title](CwtO) at (-.6,-2.875){$\left. \begin{matrix} \\ \\ \\ \end{matrix} \Jdp \right\{$};
		
		\draw [font=\scriptstyle]
				(c1) edge (r1)
				(cdots1) edge (rdots1)
				(cdots1) edge (r2)
				(cdots1) edge (rdots2)
				(cdots1) edge (r3)
				(c2) edge (r4)
				(c3) edge (r3)
				(cdots2) edge (rdots2)
				(cdots2) edge (r4)
				(cdots2) edge (r5)
				(cdots2) edge (rdots3)
				(c4) edge (r6);
	\end{tikzpicture}.
\end{center}

Let \phantomsection\label{mpDef}$\mp=\mOf(\Op)$, \phantomsection\label{mdpDef}$\mdp=\mOf(\Odp)$, \phantomsection\label{npDef}$\np=\nOf(\Op)$ and \phantomsection\label{ndpDef}$\ndp=\nOf(\Odp)$.  It is easy to see that $\mp$ denotes the number of row vertices that $\Op$ is connected to.  Then
\begin{align}
\label{jmEq}
|\Jp| + \r-1 \ = \ |\Jp| + |\J| \ \geq \ \mp.
\end{align}
Now suppose for contradiction that $|\SS(\sstar,\OO)|=1$.  By \identifiabilityThm, $\mp  \geq \np + \r$.  Substituting this into \eqref{jmEq} we obtain
\begin{align}
|\Jp| \ &\geq \ \np + 1 \label{jnOneEq}, \\
|\Jdp| \ &\geq \ \ndp + 1 \label{jnTwoEq},
\end{align}
where \eqref{jnTwoEq} follows by symmetry.

Now observe that since $\Jp$, $\Jdp$ and $\J$ form a partition of the row vertices, $\d = |\Jp| + |\Jdp| + |\J|$, so using \eqref{jnOneEq} and \eqref{jnTwoEq} we obtain
\begin{align}
\label{drEq}
\d-\r \ \geq \ \np + \ndp + 1.
\end{align}
On the other hand, since $\Op$ and $\Odp$ form a partition of the $\d-\r$ columns in $\OO$, 
\begin{align*}
\d-\r \ = \ \np + \ndp.
\end{align*}
Plugging this in \eqref{drEq}, we obtain $0 \geq 1$, which is a contradiction.  We thus conclude that $|\SS(\sstar,\OO)|>1$.

\item[($\Rightarrow$)]
For $\r=1$, we prove the converse by contrapositive.  Suppose $|\SS(\sstar,\OO)|>1$.  By \identifiabilityThm\ there exists a matrix $\Op$ formed with a subset of the columns of $\OO$ with $\mp<\np+1$.  Let $\Odp$ be the matrix formed with the remaining columns of $\OO$.

If $\mp + \mdp < \d$, there is at least one row in $\GOO$ that is disconnected, and the converse follows trivially, so suppose $\mp+\mdp = \d$.  Observe that $\np + \ndp = \d-1$.  Putting these two equations together, we obtain
\begin{align}
\label{nmEq}
\mp+\mdp - \np - 1 \ = \ \ndp.
\end{align}
Since $\mp \leq \np$, we obtain $\mdp > \ndp$.  Let $\Jp$ and $\Jdp$ be the row vertices connected to the column vertices in $\Op$ and $\Odp$ respectively.
\begin{center}
	\begin{tikzpicture}
		\node [title] (c) at (.5,-3) {$(i)$};
		\node [title] (c) at (1.25,.5) {Columns};
		\node [title] (r) at (0,.5) {Rows};
		
		\node [vertex] (c1) at (1,-0){};
		\node [vertex] (c2) at (1,-.5) {};
		\node [vertex] (c3) at (1,-1) {};
		\node [title](CwtO) at (1.55,-.5){$\left\} \begin{matrix} \\ \\ \\ \\ \end{matrix} \Op \right.$};
		
		\node [vertex] (c4) at (1,-1.75) {};
		\node [vertex] (c5) at (1,-2.25) {};
		\node [title](CwtO) at (1.55,-2){$\left\} \begin{matrix} \\ \vspace{.25cm} \\ \end{matrix} \Odp \right.$};
		
		\node [vertex] (r1) at (0,0) {};
		\node [vertex] (r2) at (0,-.5) {};
		\node [vertex] (r3) at (0,-1) {};
		\node [title](CwtO) at (-.55,-.5){$\left. \begin{matrix} \\ \\ \\ \\ \end{matrix} \Jp \right\{$};
		
		\node [vertex] (r4) at (0,-1.5) {};
		\node [vertex] (r5) at (0,-2) {};
		\node [vertex] (r6) at (0,-2.5) {};
		\node [title](CwtO) at (-.6,-2){$\left. \begin{matrix} \\ \\ \\ \\ \end{matrix} \Jdp \right\{$};
		
		\draw [font=\scriptstyle]
				(c1) edge (r1)
				(c1) edge (r2)
				(c2) edge (r2)
				(c2) edge (r3)
				(c3) edge (r1)
				(c3) edge (r3)
				(c4) edge (r4)
				(c4) edge (r5)
				(c5) edge (r5)
				(c5) edge (r6);
	\end{tikzpicture}
	\hspace{.25cm}
	\begin{tikzpicture}
		\node [title] (p1) at (0,-3) {};
		\node [title] (p2) at (0,.5) {};
		\draw [font=\scriptstyle]
				(p1) edge (p2);
	\end{tikzpicture}
	\hspace{.25cm}
	\begin{tikzpicture}
		\node [title] (c) at (.5,-3) {$(ii)$};
		\node [title] (c) at (1.25,.5) {Columns};
		\node [title] (r) at (0,.5) {Rows};
		
		\node [vertex] (c1) at (1,-0){};
		\node [vertex] (c2) at (1,-.5) {};
		\node [vertex] (c3) at (1,-1) {};
		\node [title](CwtO) at (1.55,-.5){$\left\} \begin{matrix} \\ \\ \\ \\ \end{matrix} \Op \right.$};
		
		\node [vertex] (c4) at (1,-1.75) {};
		\node [vertex] (c5) at (1,-2.25) {};
		\node [title](CwtO) at (1.55,-2){$\left\} \begin{matrix} \\ \vspace{.25cm} \\ \end{matrix} \Odp \right.$};
		
		\node [vertex] (r1) at (0,0) {};
		\node [vertex] (r2) at (0,-.5) {};
		\node [vertex] (r3) at (0,-1) {};
		\node [title](CwtO) at (-.55,-.5){$\left. \begin{matrix} \\ \\ \\ \\ \end{matrix} \Jp \right\{$};
		
		\node [vertex] (r4) at (0,-1.5) {};
		\node [vertex] (r5) at (0,-2) {};
		\node [vertex] (r6) at (0,-2.5) {};
		\node [title](CwtO) at (-.6,-2){$\left. \begin{matrix} \\ \\ \\ \\ \end{matrix} \Jdp \right\{$};
		
		\draw [font=\scriptstyle]
				(c1) edge (r1)
				(c1) edge (r2)
				(c2) edge (r2)
				(c2) edge (r3)
				(c3) edge (r1)
				(c3) edge (r3)
				(c4) edge (r3)
				(c4) edge (r4)
				(c5) edge (r4)
				(c5) edge (r5);
	\end{tikzpicture}
\end{center}
Now observe that since each column vertex only has two edges, the column vertices in $\Odp$ may connect at most $\ndp+1$ row vertices.  Since $\mdp > \ndp$, either $(i)$ the edges of $\Odp$ connect only vertices in $\Jdp$, leaving $\Jp$ and $\Jdp$ disconnected, or $(ii)$ the edges of $\Odp$ connect a vertex in $\Jp$ with a vertex in $\Jdp$, leaving at least one vertex in $\Jdp$ disconnected.  Either case, $\GOO$ is disconnected, as claimed.

For $\r=2$, consider the following sampling:
\begin{align*}
\OO=\left[ \begin{matrix}
1 & 1 & 0 & 0 & 0 & 0 & 0 & 0 \\
1 & 1 & 1 & 0 & 0 & 1 & 0 & 0 \\
0 & 1 & 1 & 0 & 0 & 0 & 1 & 0 \\
1 & 0 & 1 & 0 & 0 & 0 & 0 & 1 \\
0 & 0 & 0 & 1 & 0 & 1 & 0 & 0 \\
0 & 0 & 0 & 1 & 0 & 0 & 1 & 0 \\
0 & 0 & 0 & 1 & 0 & 0 & 0 & 1 \\
0 & 0 & 0 & 0 & 1 & 1 & 0 & 0 \\
0 & 0 & 0 & 0 & 1 & 0 & 1 & 0 \\
0 & 0 & 0 & 0 & 1 & 0 & 0 & 1
\end{matrix} \right].
\end{align*}
One may verify that $\GOO$ is $\r$-row-connected, yet it does not satisfy the conditions of \identifiabilityThm.  For instance the first $3$ columns of $\OO$, fail to satisfy \eqref{identifiabilityEq}.  This example can be easily generalized for $\r>2$. \hfill $\square$
\end{itemize}

\end{document}